\documentclass[letterpaper]{article}
\usepackage{aaai}
\usepackage{times}
\usepackage{helvet}
\usepackage{courier}
\usepackage[utf8]{inputenc}
\usepackage[english]{babel}
\usepackage{amsthm}
\usepackage{amsmath}
\usepackage{amsfonts}
\usepackage{graphicx}
\usepackage{xcolor}
\usepackage[export]{adjustbox}
\usepackage{subcaption}

\newtheorem{theorem}{Theorem}

\begin{document}
%
\title{Explaining Models by Propagating Shapley Values}
\author{Hugh Chen \\
    Paul G. Allen School of CSE \\
    University of Washington \\
    Seattle, WA, USA \\
    \And
    Scott Lundberg \\
    Microsoft Research \\
    Redmond, WA, USA \\
    \And
    Su-In Lee \\
    Paul G. Allen School of CSE \\
    University of Washington \\
    Seattle, WA, USA \\
}
\maketitle
\begin{abstract}
\begin{quote}
In healthcare, making the best possible predictions with complex models (e.g., neural networks, ensembles/stacks of different models) can impact patient welfare.  In order to make these complex models explainable, we present DeepSHAP for mixed model types, a framework for layer wise propagation of Shapley values that builds upon DeepLIFT (an existing approach for explaining neural networks).  We show that in addition to being able to explain neural networks, this new framework naturally enables attributions for stacks of mixed models (e.g., neural network feature extractor into a tree model) as well as attributions of the loss.  Finally, we theoretically justify a method for obtaining attributions with respect to a background distribution (under a Shapley value framework).
\end{quote}
\end{abstract}

\section{Introduction}

Neural networks and ensembles of models are currently used across many domains.  For these complex models, explanations accounting for how features relate to predictions is often desirable and at times mandatory \cite{goodman2017european}.  In medicine, explainable AI (XAI) is important for scientific discovery, transparency, and much more \cite{holzinger2017we}.  One popular class of XAI methods is per-sample feature attributions (i.e., values for each feature for a given prediction).  

In this paper, we focus on SHAP values \cite{lundberg2017unified} -- Shapley values \cite{shapley1953value} with a conditional expectation of the model prediction as the set function.  Shapley values are the only additive feature attribution method that satisfies the desirable properties of local accuracy, missingness, and consistency.  In order to approximate SHAP values for neural networks, we fix a problem in the original formulation of DeepSHAP \cite{lundberg2017unified} where previously it used $E[x]$ as the reference and theoretically justify a new method to create explanations relative to background distributions.  Furthermore, we extend it to explain stacks of mixed model types as well as loss functions rather than margin outputs. 



Popular model agnostic explanation methods that also aim to obtain SHAP values are KernelSHAP \cite{lundberg2017unified} and IME \cite{vstrumbelj2014explaining}.  The downside of most model agnostic methods are that they are sampling based and consequently high variance or slow.

Alternatively, local feature attributions targeted to deep networks has been addressed in numerous works: Occlusion \cite{zeiler2014visualizing}, Saliency Maps \cite{simonyan2013deep}, Layer-Wise Relevance Propagation \cite{bach2015pixel}, DeepLIFT, Integrated Gradients (IG) \cite{sundararajan2017axiomatic}, and Generalized Integrated Gradients (GIG) \cite{gig}.  

Of these methods, the ones that have connections to the Shapley Values are IG and GIG.  IG integrates gradients along a path between a baseline and the sample being explained.  This explanation approaches the Aumann-Shapley value.  GIG is a generalization of IG to explain losses and mixed model types -- a feature DeepSHAP also aims to provide.  IG and GIG have two downsides: 1.) integrating along a path can be expensive or imprecise and 2.) the Aumann-Shapley values fundamentally differ to the SHAP values we aim to approximate.  Finally, DASP \cite{ancona2019explaining} is an approach that approximates SHAP values for deep networks.  This approach works by replacing point activations at all layers by probability distributions and requires many more model evaluations than DeepSHAP.  Because DASP aims to obtain the same SHAP values as in DeepSHAP it is possible to use DASP as a part of the DeepSHAP framework.


\section{Approach}

\subsection{Propagating SHAP values}

\begin{figure}[!ht]
\centering
\includegraphics[width=.8\linewidth]{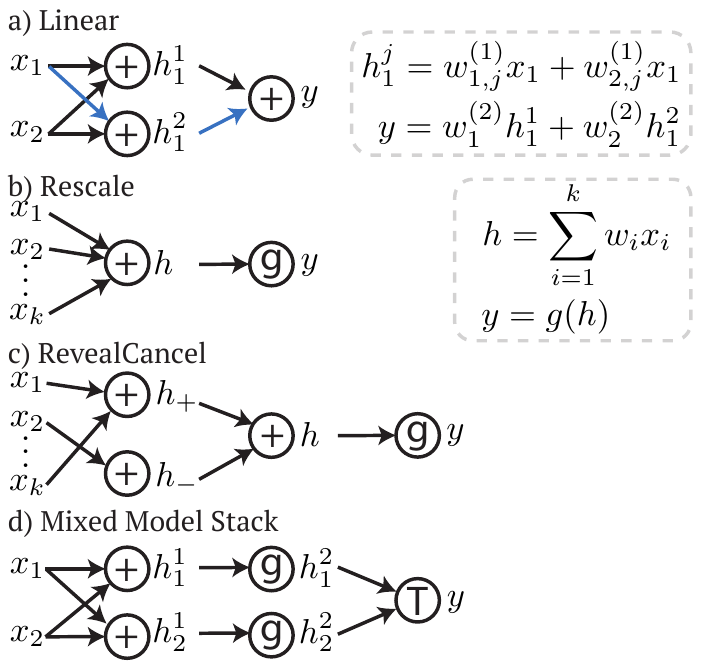}
\vspace{-.2cm}
\caption{\textit{Visualization of models for understanding DeepLIFT's connection to SHAP values}.  In the figure $g$ is a non-linear function and $T$ is a non-differentiable tree model.
}
\vspace{-.5cm}
\label{fig:understanding}
\end{figure}

DeepSHAP builds upon DeepLIFT; in this section we aim to better understand how DeepLIFT's rules connect to SHAP values.  This has been briefly touched upon in \cite{deeplift} and \cite{lundberg2017unified}, but here we explicitly define the relationship.

DeepSHAP is a method that explains a sample (foreground sample), by setting features to be ``missing''.  Missing features are set to corresponding values in a baseline sample (background sample).  Note that DeepSHAP generally uses a background distribution, however focusing on a single background sample is sufficient because we can rewrite the SHAP values as an average over attributions with respect to a single background sample at a time (see next section for more details).
In this section, we define a foreground sample to have features x$f_{x_i}$ and neuron values $f_{h}$ (obtained by a forward pass) and a background sample to have  $b_{x_i}$ or $b_{h}$.  Finally we define $\phi(\cdot)$ to be attribution values.  

If our model is \textbf{fully linear} as in Figure \ref{fig:understanding}a, we can get the exact SHAP values for an input $x_i$ by summing the attributions along all possible paths between that input $x_i$ and the model's output $y$.  Therefore, we can focus on a particular path (in blue).  Furthermore, the path's contribution to $\phi(x_i)$ is exactly the product of the weights along the path and the difference in $x_1$: $w^{(2)}_{2} w^{(1)}_{1,2} (f_{x_1}-b_{x_1})$, because we can rewrite the layers of linear equations in \ref{fig:understanding}a as a single linear equation.  Note that we can derive the attribution for $x_1$ in terms of the attribution of intermediary nodes (as in the chain rule):
\begin{align}
\phi(h_1^2)&=w_2^{(2)}(f_{h_1^2}{-}b_{h_1^2})\nonumber\\
\phi(x_1)&=\frac{\phi(h_1^2)}{f_{h_1^2}{-}b_{h_1^2}}w_{1,2}^{(1)}(f_{x_1}{-}b_{x_1})\label{eq:phi_lin_1}
\end{align}

Next, we move on to reinterpreting the two variants of DeepLIFT: the Rescale rule and the RevealCancel rule.  First, a gradient based interpretation of the \textbf{Rescale rule} has been discussed in \cite{ancona2018towards}.  Here, we explicitly tie this interpretation to the SHAP values we hope to obtain.  

For clarity, we consider the example in Figure \ref{fig:understanding}b.  First, the attribution value for $\phi(h)$ is $g(f_h)-g(b_h)$ because SHAP values maintain local accuracy (sum of attributions equals $f_y-b_y$) and $g$ is a function with a single input.   Then, under the Rescale rule, $\phi(x_i){=}\frac{\phi(h)}{f_h-b_h}w_i(f_{x_i}-b_{x_i})$ (note the resemblance to Equation (\ref{eq:phi_lin_1})). Under this formulation it is easy to see that the Rescale rule first computes the exact SHAP value for $h$ and then propagates it back linearly.  In other words, the the non-linear and linear functions are treated as separate functions.  Passing back nonlinear attributions linearly is clearly an approximation, but confers two benefits: 1.) fast computation on order of a backward pass and 2.) a guarantee of local accuracy.  


Next, we describe how the \textbf{RevealCancel rule} (originally formulated to bring DeepLIFT closer to SHAP values) connects to SHAP values in the context of Figure \ref{fig:understanding}c.  RevealCancel partitions $x_i$ into positive and negative components based on if $w_i(f_{x_i}-b_{x_i})<t$ (where $t{=}0$), in essence forming nodes $h_+$ and $h_-$.  This rule computes the exact SHAP attributions for $h_+$ and $h_-$ and then propagates the resultant SHAP values linearly.  Specifically: 
\begin{align*}
\phi(g_+)&=\frac{1}{2}((g(f_{h_+}{+}f_{h_-})-g(b_{h_+}{+}f_{h_-})+\nonumber\\&(g(f_{h_+}{+}b_{h_-})-g(b_{h_+}{+}b_{h_-}))\\    
\phi(g_+)&=\frac{1}{2}((g(f_{h_+}{+}f_{h_-})-g(f_{h_+}{+}b_{h_-})+\nonumber\\&(g(b_{h_+}{+}f_{h_-})-g(b_{h_+}{+}b_{h_-}))\\
\phi(x_i)&=
\begin{cases}
  \frac{\phi_{h_+}}{f_{h_+}-b_{h_+}}w_i(f_{x_i}-b_{x_i}),   & \text{if}\ w_i(f_{x_i}-b_{x_i})>t \\
  \frac{\phi_{h_-}}{f_{h_-}-b_{h_-}}w_i(f_{x_i}-b_{x_i}),   & \text{otherwise}\\
\end{cases}
\end{align*}

Under this formulation, we can see that in contrast to the Rescale rule that explains a linearity and nonlinearity by exactly explaining the nonlinearity and backpropagating, the RevealCancel rule exactly explains the nonlinearity and a partition of the inputs to the linearity as a single function prior to backpropagating.  The RevealCancel rule incurs a higher computational cost in order to get a an estimate of $\phi(x_i)$ that is ideally closer to the SHAP values.


This reframing naturally motivates explanations for \textbf{stacks of mixed model types}.  In particular, for Figure \ref{fig:understanding}d, we can take advantage of fast, exact methods for obtaining SHAP values for tree models to obtain $\phi(h_j^2)$ using Independent Tree SHAP \cite{treeshap}.  Then, we can propagate these attributions to get $\phi(x_i)$ using either the Rescale or RevealCancel rule.  This argument extends to explaining losses rather than output margins as well.

Although we consider specific examples here, the linear propagation described above will generalize to arbitrary networks if SHAP values can be computed or approximated for individual components.

\subsection{SHAP values with a background distribution}
\label{sec:theory_multref}

Note that many methods (Integrated Gradients, Occlusion) recommend the utilization of a single background/reference sample.  In fact, DeepSHAP as previously described in \cite{lundberg2017unified} created attributions with respect to a single reference equal to the expected value of the inputs.  However, in order to obtain SHAP values for a given background distribution, we prove that the correct approach is as follows: obtain SHAP values for each baseline in your background distribution and average over the resultant attributions.  Although similar methodologies have been used heuristically \cite{deeplift,erion2019learning}, we provide a theoretical justification in Theorem 1 in the context of SHAP values.

\begin{theorem}
\label{thm}
The average over single reference SHAP values approaches the true SHAP values for a given distribution.
\end{theorem}
\begin{proof}
Define $D$ to be the data distribution, $N$ to be the set of all features, and $f$ to be the model being explained.  Additionally, define  $\mathcal{X}(x,x',S)$ to return a sample where the features in $S$ are taken from $x$ and the remaining features from $x'$.  Define $C$ to be all combinations of the set $N \setminus \{i\}$ and $P$ to be all permutations of $N \setminus \{i\}$.  Starting with the definition of SHAP values for a single feature: $\phi_i(x)$
\begin{align*}
&= \sum_{S\in C } W(|S|,|N|)(\mathbb{E}_{D}[f(X)|x_{S\cup \{i\}}] {-} \mathbb{E}_{D}[f(X)|x_{S}])\\
&=\frac{1}{|P|}\sum_{S\subseteq P} \mathbb{E}_\mathcal{D}[f(x)|\text{do}(x_{S \cup \{i\}})] {-} \mathbb{E}_\mathcal{D}[\text{do}(f(x)|x_{S})]\\
&= \frac{1}{|P|}\sum_{S\subseteq P}\frac{1}{|D|}\sum_{x'\in D} f(\mathcal{X}(x,x',S\cup \{i\})) {-} f(\mathcal{X}(x,x',S))
\\
&= \frac{1}{|D|}\sum_{x'\in D} \underbrace{\frac{1}{|P|}\sum_{S\subseteq P} f(\mathcal{X}(x,x',S\cup \{i\})) {-} f(\mathcal{X}(x,x',S))}_\text{single reference SHAP value}
\end{align*}
where the second step depends on an interventional conditional expectation \cite{janzing2019feature} which is very close to Random Baseline Shapley in \cite{sundararajan2019many}).
\end{proof}

\section{Experiments}
\label{sec:experiments}



\subsection{Background distributions avoid bias}
\label{sec:experiments:results:multiple_references}
\begin{figure}[!ht]
\centering
\vspace{-.2cm}
\includegraphics[width=.8\linewidth]{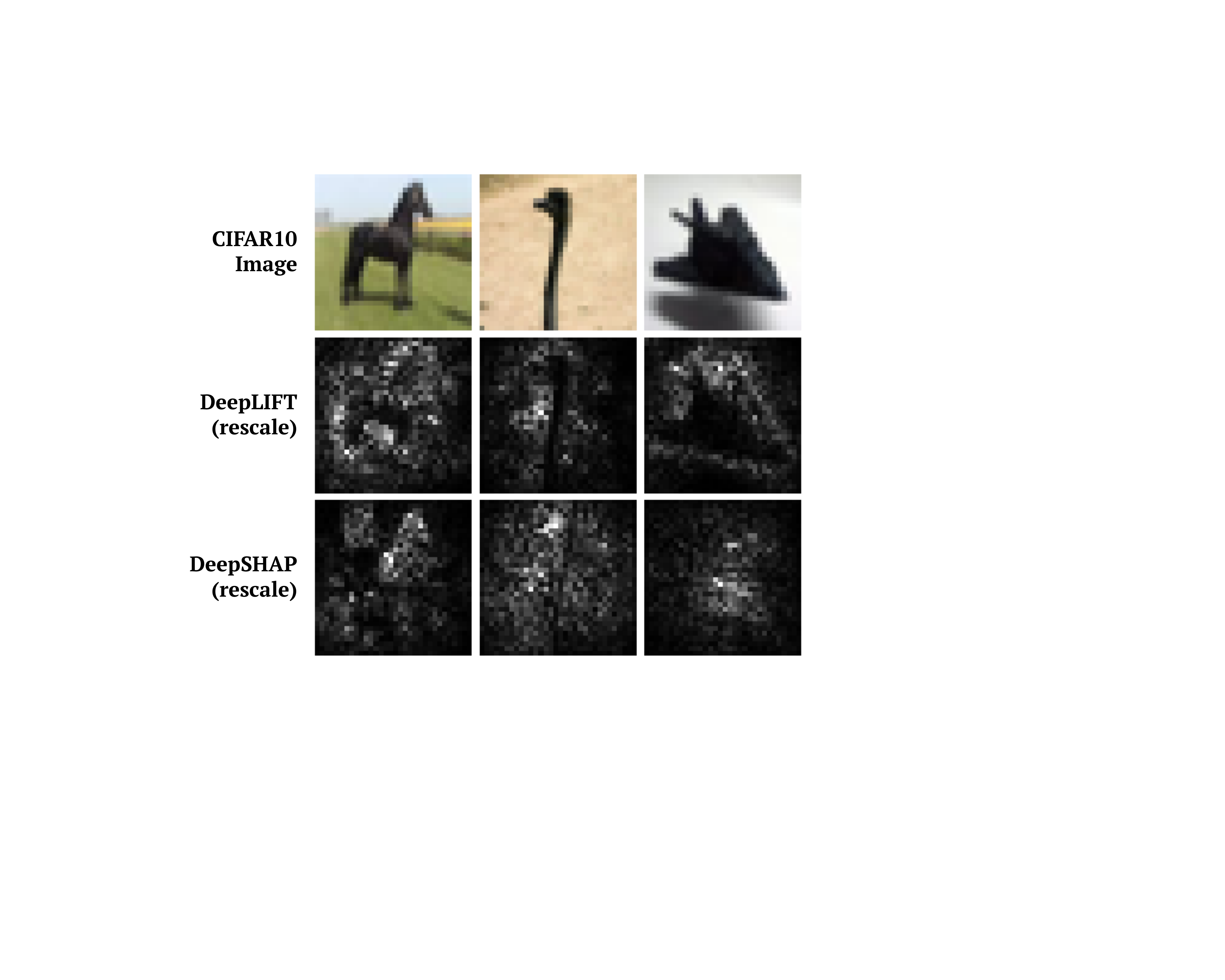}
\vspace{-.2cm}
\caption{\textit{Using a single baseline leads to bias in explanations.}
}
\label{fig:reference}
\end{figure}

\noindent
In this section, we utilize the popular CIFAR10 dataset \cite{krizhevsky2009learning} to demonstrate that single references lead to bias in explanations.  We train a CNN that achieves 75.56\% test accuracy and evaluate it using either a zero baseline as in DeepLIFT or with a random set of 1000 baselines as in DeepSHAP.  

In Figure \ref{fig:reference}, we can see that for these images drawn from the CIFAR10 training set, DeepLIFT has a clear bias that results in low attributions for darker regions of the image.  For DeepSHAP, having multiple references drawn from a background distribution solves this problem and we see attributions in sensical dark regions in the image.

\subsection{Explaining mortality prediction}

\noindent
In this section, we validate DeepSHAP's explanations for an MLP with 82.56\% test accuracy predicting 15 year mortality.  The dataset has 79 features for 14,407 individuals released by \cite{treeshap} based on NHANES I Epidemiologic Followup Study \cite{cox1997plan}.  

\begin{figure}[!ht]
\centering
\vspace{-.3cm}
\includegraphics[width=\linewidth]{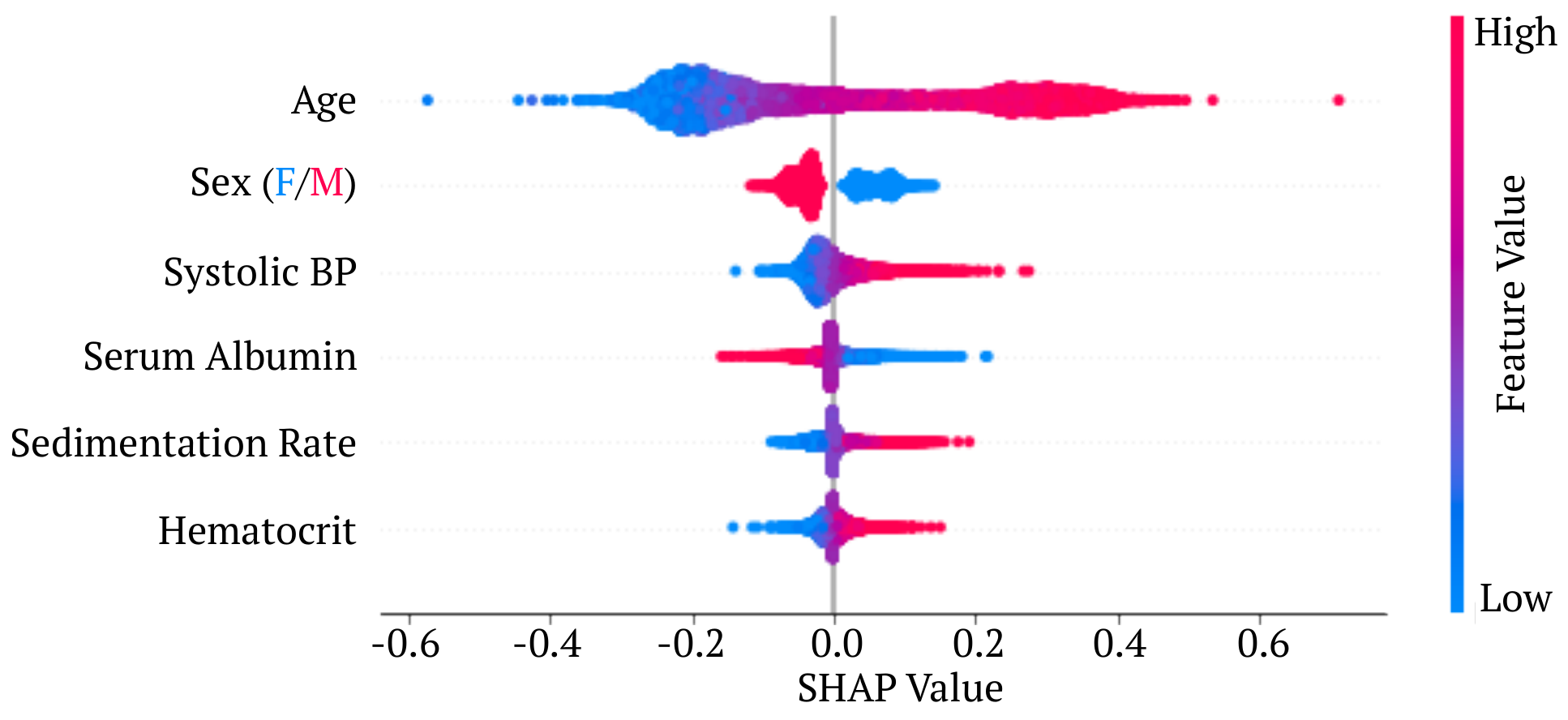}
\vspace{-.5cm}
\caption{\textit{Summary plot of DeepSHAP attribution values.} Each point is the local feature attribution value, colored by feature value.  For brevity, we only show the top 6 features.}
\label{fig:mortality_nhanes}
\end{figure}

In Figure \ref{fig:mortality_nhanes}, we plot a summary of DeepSHAP (with 1000 random background samples) attributions for all NHANES training samples (n=$8023$) and notice a few trends.  First, Age is predictably the most important and old age contributes to a positive mortality prediction (positive SHAP values).  Second, the Sex feature validates a well-known difference in mortality \cite{gjoncca1999male}.  Finally, the trends linking high systolic BP, low serum albumin, high sedimentation rate, and high hematocrit to mortality have been independently discovered  \cite{port2000systolic,goldwasser1997serumalbumin,paul2012hematocrit,go2016sedimentation}.

\begin{figure}[!ht]
\centering
\includegraphics[width=\linewidth]{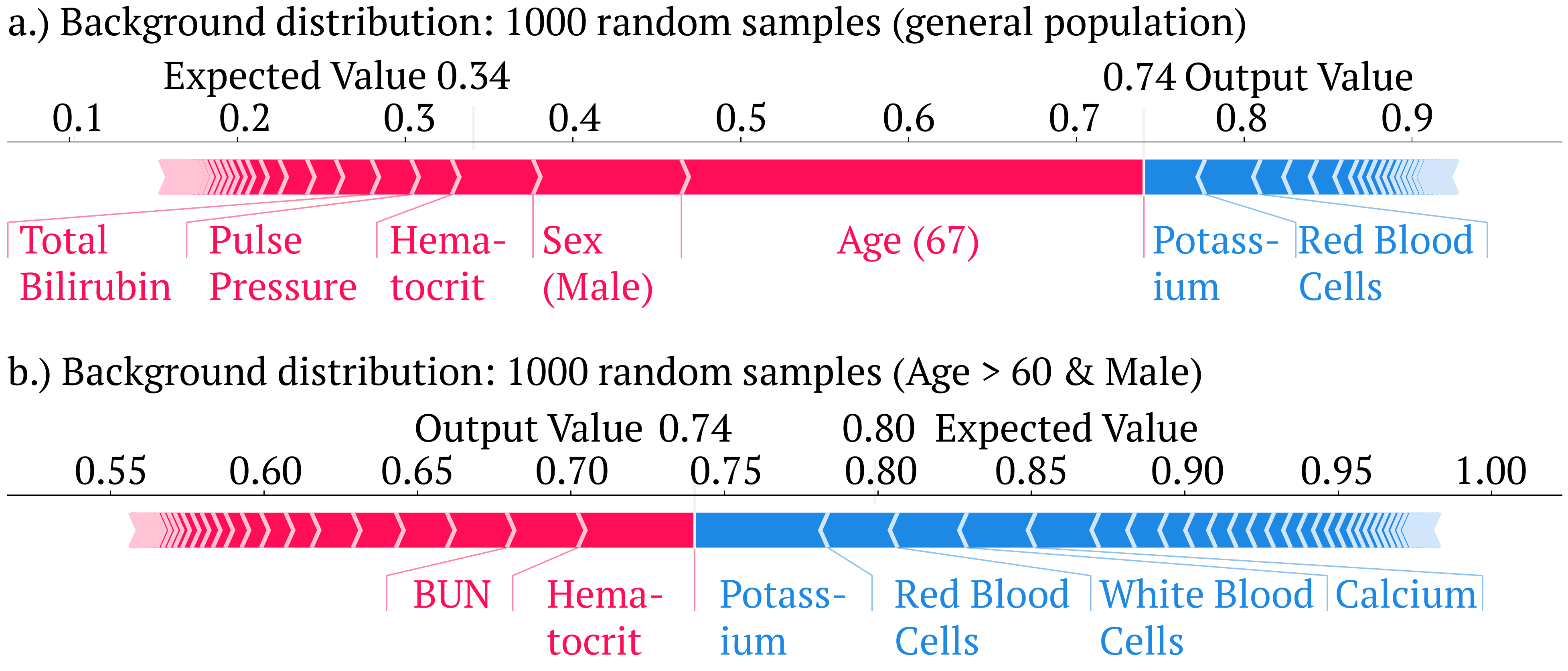}
\vspace{-.5cm}
\caption{\textit{Explaining an individual's mortality prediction for different backgrounds distributions.}}
\label{fig:ind_diff_backgrounds}
\end{figure}

Next, we show the benefits of being able to specify a background distribution.  In Figure  \ref{fig:ind_diff_backgrounds}a, we see that explaining an individual's mortality prediction with respect to a general population emphasizes that the individual's age and gender are driving a high mortality prediction.  However, in practice doctors are unlikely to compare a 67-year old male to a general population that includes much younger individuals.  In Figure \ref{fig:ind_diff_backgrounds}b, being able to specify a background distribution allows us to compare our individual against a more relevant distribution of males over 60.  In this case, gender and age are naturally no longer important, and the individual actually may not have cause for concern.

\subsection{Interpreting a stack of mixed model types}
\label{sec:experiments:results:mixed_model_types}

\begin{figure}[!ht]
\centering
\vspace{-.25cm}
\includegraphics[width=\linewidth]{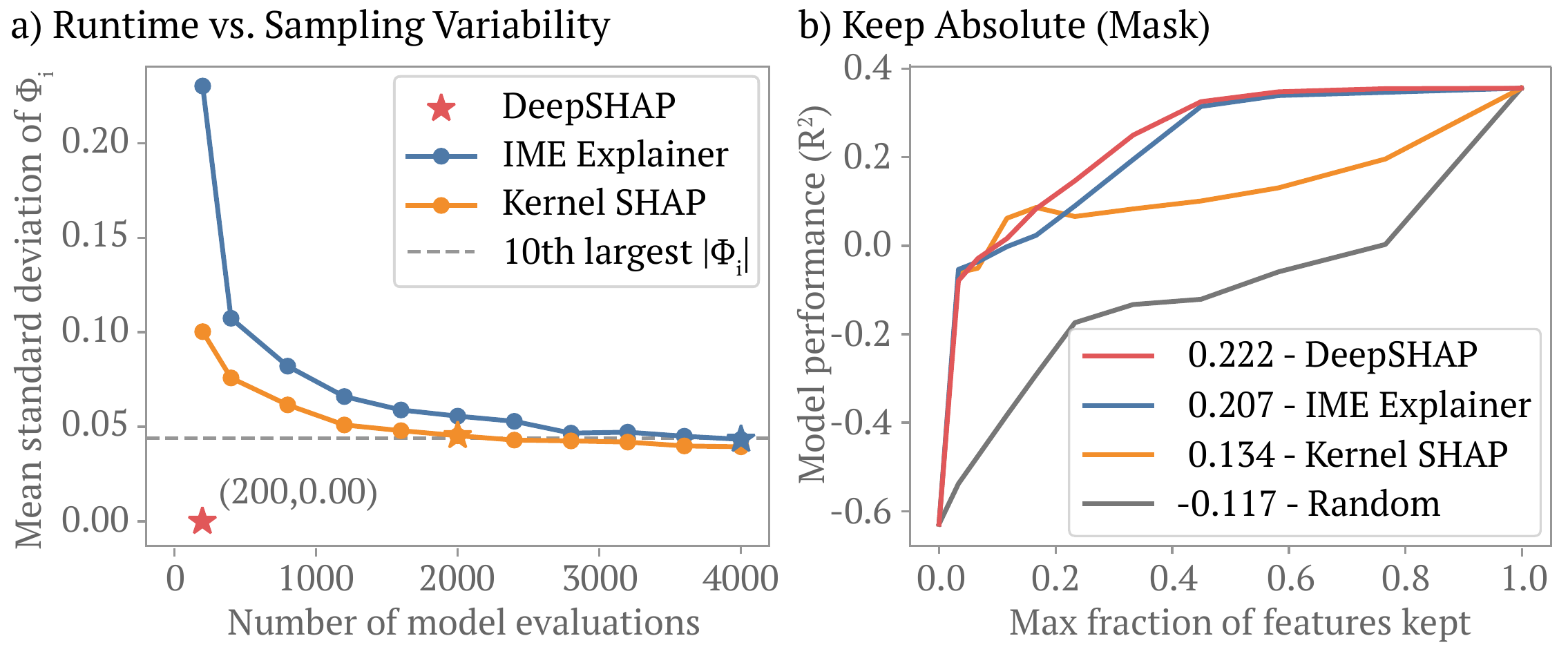}
\caption{\textit{Ablation test for explaining an LSTM feature extractor fed into an XGB model.}  All methods used background of 20 samples obtained via kmeans.  [a.] Convergence of methods for a single explanation.  [b.] Model performance versus \# features kept for DeepSHAP (rescale), IME Explainer (4000 samples), KernelSHAP (2000 samples) and a baseline (Random) (AUC in the legend).
}
\vspace{-.2cm}
\label{fig:ablation}
\end{figure}

\noindent
Stacks, and more generally ensembles, of models are increasingly popular for performant predictions \cite{bao2009stacking,gunecs2017stacked,zhai2018development}.  In this section, our aim is to evaluate the efficacy of DeepSHAP for a neural network feature extractor fed into a tree model.  For this experiment, we use the Rescale rule for simplicity and Independent TreeSHAP to explain the tree model \cite{treeshap}.  The dataset is a simulated one called Corrgroups60.  Features $X\in\mathbb{R}^{1000\times60}$ have tight correlation between groups of features ($x_i$ is feature $i$), where $\rho_{x_i,x_i}{=}1$, $\rho_{x_i,x_{i+1}}{=}\rho_{x_i,x_{i+2}}{=}\rho_{x_{i+1},x_{i+2}}{=}.99$ if $(i \bmod 3) {=} 0$, and $\rho_{x_i,x_j}{=}0$ otherwise.  The label $y\in \mathbb{R}^{n}$ is generated linearly as $y{=}X\beta {+} \epsilon$ where $\epsilon {\sim} \mathcal{N}_{n}(\mu{=}0,\sigma^2{=}10^{-4})$ and $\beta_i{=}1$ if $(i \bmod 3) {=} 0$ and $\beta_i{=}0$ otherwise.


We evaluate DeepSHAP with an ablation metric called \textit{keep absolute (mask)} \cite{treeshap}.  The metric works in the following manner:  1) Obtain the feature attributions for all test samples  2) Mask all features (by mean imputation) 3) Introduce one feature at a time (unmask) from largest absolute attribution value to smallest for each sample and measure $R^2$.  The $R^2$ should initially increase rapidly, because we introduce the ``most important'' features first.   


We compare against two sampling-based methods (a natural alternative for explaining mixed model stacks) that provide SHAP values in expectation: KernelSHAP and IME explainer.  In Figure \ref{fig:ablation}b, DeepSHAP (rescale) has no variability and requires a fixed number of model evaluations.  IME Explainer and KernelSHAP, benefit from having more samples (and therefore more model evaluations).  For the final comparison, we check the variability of the tenth largest attribution (absolute value) of the sampling based methods to determine ``convergence'' across different numbers of samples.  Then, we use the number of samples at the point of ``convergence'' for the next figure.

In Figure \ref{fig:ablation}c, we can see that DeepSHAP has a slightly higher performance than model agnostic methods.  Promisingly, all methods demonstrate initial steepness in their performance; this indicates that the most important features had higher attribution values.  We hypothesize that KernelSHAP and IME Explainer's lower performance is due in part to noise in their estimates.  This highlights an important point: model agnostic methods often have sampling variability that makes determining convergence difficult.  For a fixed background distribution, DeepSHAP does not suffer from this variability and generally requires fewer model evaluations.

\subsection{Improving the RevealCancel rule}
\label{sec:experiments:results:rescale_revealcancel}

\begin{figure}[!ht]
\centering
\vspace{-.3cm}
\includegraphics[width=\linewidth]{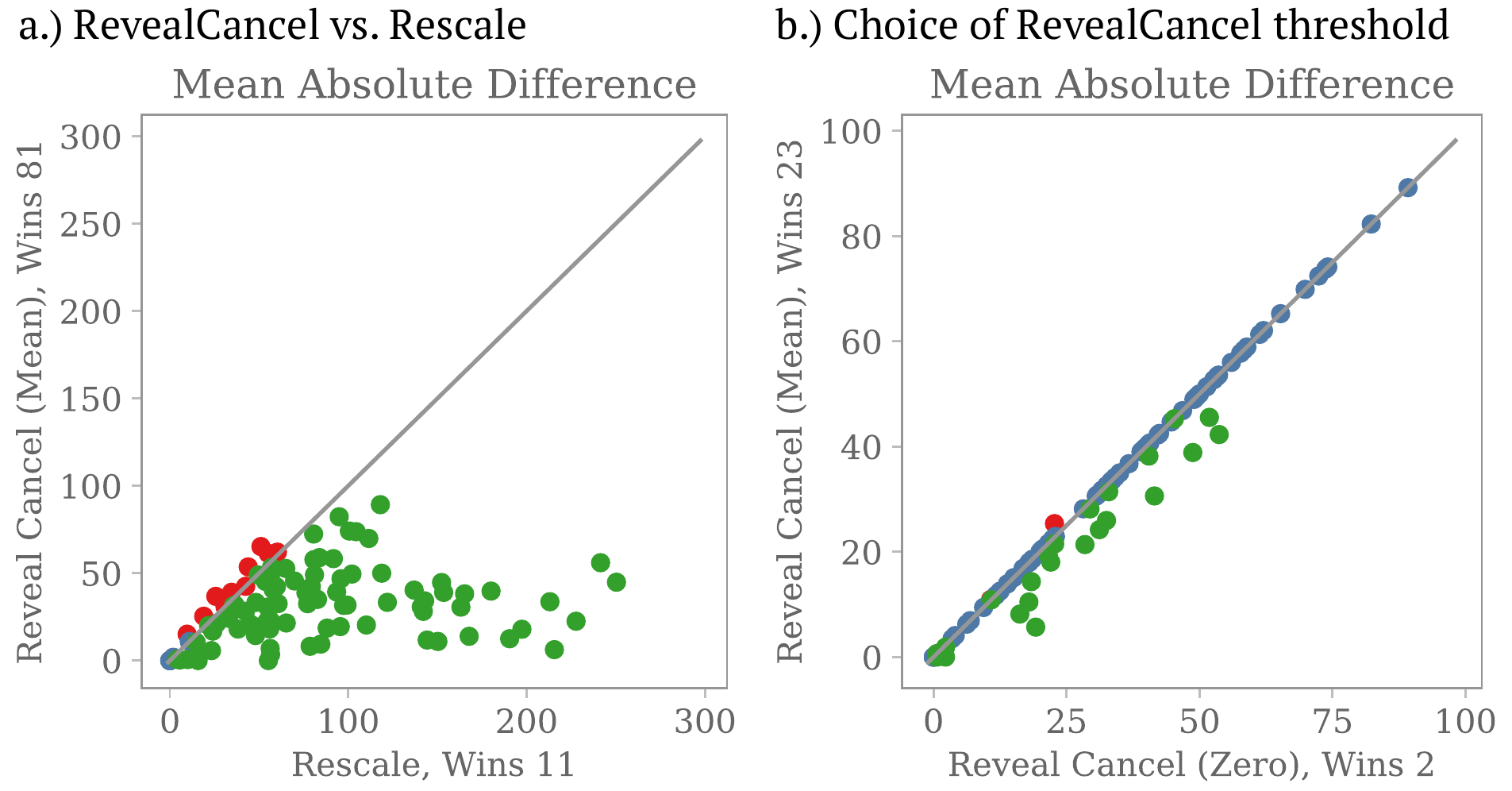}
\caption{\textit{Comparison of new RevealCancel$^\text{Mean}$ rule for estimating SHAP values on a toy example.}  The axes correspond to mean absolute difference from the SHAP values (computed exactly).  Green means RevealCancel$^\text{Mean}$ wins and red means it loses.}
\label{fig:revealcancel}
\end{figure}

\noindent
DeepLIFT's RevealCancel rule's connection to the SHAP values is touched upon in \cite{deeplift}.  Our SHAP value framework explicitly defines this connection.  In this section, we propose a simple improvement to the RevealCancel rule.  In DeepLIFT's RevealCancel rule the threshold $t$ is set to $0$ (for splitting $h_-$ and $h_+$).  Our proposed rule RevealCancel$^\text{Mean}$ sets the threshold to the mean value of $w_i(f_{x_i}{-}b_{x_i})$ across $i$.  Intuitively, splitting by the mean better separates $x_i$ nodes, resulting in a better approximation than splitting by zero.

We experimentally validate RevealCancel$^\text{Mean}$ in Figure \ref{fig:revealcancel}, explaining a simple function: $\text{ReLU}(x_1+x_2+x_3+x_4+100)$.  We fix the background to zero: $b_{x_i}{=}0$ and draw 100 foreground samples from a discrete uniform distribution: $f_{x_i}{\sim} U\{-1000,1000\}$.  

In Figure \ref{fig:revealcancel}a, we show that RevealCancel$^\text{Mean}$ offers a large improvement for approximating SHAP values over the Rescale rule and a modest one over the original RevealCancel rule (at no additional asymptotic computational cost).  

\section{Conclusion}

In this paper, we improve the original DeepSHAP formulation \cite{lundberg2017unified} in several ways: we 1.) provide a new theoretically justified way to provide attributions with a background distribution 2.) extend DeepSHAP to explain stacks of mixed model types 3.) present improvements of the RevealCancel rule.

Future work includes more quantitative validation on different data sets and comparison to more interpretability methods.  In addition, we primarily used Rescale rule for many of these evaluations, but more empirical evaluations of RevealCancel are also important.


\bibliography{aaai}

\begin{thebibliography}{}

\bibitem[\protect\citeauthoryear{Ancona \bgroup et al\mbox.\egroup
  }{2018}]{ancona2018towards}
Ancona, M.; Ceolini, E.; Oztireli, C.; and Gross, M.
\newblock 2018.
\newblock Towards better understanding of gradient-based attribution methods
  for deep neural networks.
\newblock In {\em 6th International Conference on Learning Representations
  (ICLR 2018)}.

\bibitem[\protect\citeauthoryear{Ancona, {\"O}ztireli, and
  Gross}{2019}]{ancona2019explaining}
Ancona, M.; {\"O}ztireli, C.; and Gross, M.
\newblock 2019.
\newblock Explaining deep neural networks with a polynomial time algorithm for
  shapley values approximation.
\newblock {\em arXiv preprint arXiv:1903.10992}.

\bibitem[\protect\citeauthoryear{Bach \bgroup et al\mbox.\egroup
  }{2015}]{bach2015pixel}
Bach, S.; Binder, A.; Montavon, G.; Klauschen, F.; M{\"u}ller, K.-R.; and
  Samek, W.
\newblock 2015.
\newblock On pixel-wise explanations for non-linear classifier decisions by
  layer-wise relevance propagation.
\newblock {\em PloS one} 10(7):e0130140.

\bibitem[\protect\citeauthoryear{Bao, Bergman, and
  Thompson}{2009}]{bao2009stacking}
Bao, X.; Bergman, L.; and Thompson, R.
\newblock 2009.
\newblock Stacking recommendation engines with additional meta-features.
\newblock In {\em Proceedings of the third ACM conference on Recommender
  systems},  109--116.
\newblock ACM.

\bibitem[\protect\citeauthoryear{Cox \bgroup et al\mbox.\egroup
  }{1997}]{cox1997plan}
Cox, C.~S.; Feldman, J.~J.; Golden, C.~D.; Lane, M.~A.; Madans, J.~H.;
  Mussolino, M.~E.; and Rothwell, S.~T.
\newblock 1997.
\newblock Plan and operation of the nhanes i epidemiologic followup study,
  1992.
\newblock {\em Vital and Health Statistics}.

\bibitem[\protect\citeauthoryear{Erion \bgroup et al\mbox.\egroup
  }{2019}]{erion2019learning}
Erion, G.; Janizek, J.~D.; Sturmfels, P.; Lundberg, S.; and Lee, S.-I.
\newblock 2019.
\newblock Learning explainable models using attribution priors.
\newblock {\em arXiv preprint arXiv:1906.10670}.

\bibitem[\protect\citeauthoryear{Gjon{\c{c}}a \bgroup et al\mbox.\egroup
  }{1999}]{gjoncca1999male}
Gjon{\c{c}}a, A.; Tomassini, C.; Vaupel, J.~W.; et~al.
\newblock 1999.
\newblock {\em Male-female differences in mortality in the developed world}.
\newblock Citeseer.

\bibitem[\protect\citeauthoryear{Go \bgroup et al\mbox.\egroup
  }{2016}]{go2016sedimentation}
Go, D.~J.; Lee, E.~Y.; Lee, E.~B.; Song, Y.~W.; Konig, M.~F.; and Park, J.~K.
\newblock 2016.
\newblock Elevated erythrocyte sedimentation rate is predictive of interstitial
  lung disease and mortality in dermatomyositis: a korean retrospective cohort
  study.
\newblock {\em Journal of Korean medical science} 31(3):389--396.

\bibitem[\protect\citeauthoryear{Goldwasser and
  Feldman}{1997}]{goldwasser1997serumalbumin}
Goldwasser, P., and Feldman, J.
\newblock 1997.
\newblock Association of serum albumin and mortality risk.
\newblock {\em Journal of clinical epidemiology} 50(6):693--703.

\bibitem[\protect\citeauthoryear{Goodman and
  Flaxman}{2017}]{goodman2017european}
Goodman, B., and Flaxman, S.
\newblock 2017.
\newblock European union regulations on algorithmic decision-making and a
  “right to explanation”.
\newblock {\em AI Magazine} 38(3):50--57.

\bibitem[\protect\citeauthoryear{G{\"u}ne{\c{s}}, Wolfinger, and
  Tan}{2017}]{gunecs2017stacked}
G{\"u}ne{\c{s}}, F.; Wolfinger, R.; and Tan, P.-Y.
\newblock 2017.
\newblock Stacked ensemble models for improved prediction accuracy.
\newblock In {\em SAS Conference Proceedings}.

\bibitem[\protect\citeauthoryear{Holzinger \bgroup et al\mbox.\egroup
  }{2017}]{holzinger2017we}
Holzinger, A.; Biemann, C.; Pattichis, C.~S.; and Kell, D.~B.
\newblock 2017.
\newblock What do we need to build explainable ai systems for the medical
  domain?
\newblock {\em arXiv preprint arXiv:1712.09923}.

\bibitem[\protect\citeauthoryear{Janzing, Minorics, and
  Bl{\"o}baum}{2019}]{janzing2019feature}
Janzing, D.; Minorics, L.; and Bl{\"o}baum, P.
\newblock 2019.
\newblock Feature relevance quantification in explainable ai: A causality
  problem.
\newblock {\em arXiv preprint arXiv:1910.13413}.

\bibitem[\protect\citeauthoryear{Krizhevsky, Hinton, and
  others}{2009}]{krizhevsky2009learning}
Krizhevsky, A.; Hinton, G.; et~al.
\newblock 2009.
\newblock Learning multiple layers of features from tiny images.
\newblock Technical report, Citeseer.

\bibitem[\protect\citeauthoryear{Lundberg and Lee}{2017}]{lundberg2017unified}
Lundberg, S.~M., and Lee, S.-I.
\newblock 2017.
\newblock A unified approach to interpreting model predictions.
\newblock In {\em Advances in Neural Information Processing Systems},
  4765--4774.

\bibitem[\protect\citeauthoryear{Lundberg \bgroup et al\mbox.\egroup
  }{2018}]{treeshap}
Lundberg, S.~M.; Erion, G.; Chen, H.; DeGrave, A.; Prutkin, J.~M.; Nair, B.;
  Katz, R.; Himmelfarb, J.; Bansal, N.; and Lee, S.
\newblock 2018.
\newblock Explainable ai for trees: From local explanations to global
  understanding.
\newblock {\em CoRR} abs/1905.04610.

\bibitem[\protect\citeauthoryear{Merrill \bgroup et al\mbox.\egroup
  }{2019}]{gig}
Merrill, J.; Ward, G.; Kamkar, S.; Budzik, J.; and Merrill, D.
\newblock 2019.
\newblock Generalized integrated gradients: A practical method for explaining
  diverse ensembles.
\newblock {\em CoRR} abs/1909.01869.

\bibitem[\protect\citeauthoryear{Paul \bgroup et al\mbox.\egroup
  }{2012}]{paul2012hematocrit}
Paul, L.; Jeemon, P.; Hewitt, J.; McCallum, L.; Higgins, P.; Walters, M.;
  McClure, J.; Dawson, J.; Meredith, P.; Jones, G.~C.; et~al.
\newblock 2012.
\newblock Hematocrit predicts long-term mortality in a nonlinear and
  sex-specific manner in hypertensive adults.
\newblock {\em Hypertension} 60(3):631--638.

\bibitem[\protect\citeauthoryear{Port \bgroup et al\mbox.\egroup
  }{2000}]{port2000systolic}
Port, S.; Demer, L.; Jennrich, R.; Walter, D.; and Garfinkel, A.
\newblock 2000.
\newblock Systolic blood pressure and mortality.
\newblock {\em The Lancet} 355(9199):175--180.

\bibitem[\protect\citeauthoryear{Shapley}{1953}]{shapley1953value}
Shapley, L.~S.
\newblock 1953.
\newblock A value for n-person games.
\newblock {\em Contributions to the Theory of Games} 2(28):307--317.

\bibitem[\protect\citeauthoryear{Shrikumar, Greenside, and
  Kundaje}{2017}]{deeplift}
Shrikumar, A.; Greenside, P.; and Kundaje, A.
\newblock 2017.
\newblock Learning important features through propagating activation
  differences.
\newblock In {\em Proceedings of the 34th International Conference on Machine
  Learning-Volume 70},  3145--3153.
\newblock JMLR. org.

\bibitem[\protect\citeauthoryear{Simonyan, Vedaldi, and
  Zisserman}{2013}]{simonyan2013deep}
Simonyan, K.; Vedaldi, A.; and Zisserman, A.
\newblock 2013.
\newblock Deep inside convolutional networks: Visualising image classification
  models and saliency maps.
\newblock {\em arXiv preprint arXiv:1312.6034}.

\bibitem[\protect\citeauthoryear{{\v{S}}trumbelj and
  Kononenko}{2014}]{vstrumbelj2014explaining}
{\v{S}}trumbelj, E., and Kononenko, I.
\newblock 2014.
\newblock Explaining prediction models and individual predictions with feature
  contributions.
\newblock {\em Knowledge and information systems} 41(3):647--665.

\bibitem[\protect\citeauthoryear{Sundararajan and
  Najmi}{2019}]{sundararajan2019many}
Sundararajan, M., and Najmi, A.
\newblock 2019.
\newblock The many shapley values for model explanation.
\newblock {\em arXiv preprint arXiv:1908.08474}.

\bibitem[\protect\citeauthoryear{Sundararajan, Taly, and
  Yan}{2017}]{sundararajan2017axiomatic}
Sundararajan, M.; Taly, A.; and Yan, Q.
\newblock 2017.
\newblock Axiomatic attribution for deep networks.
\newblock {\em arXiv preprint arXiv:1703.01365}.

\bibitem[\protect\citeauthoryear{Zeiler and
  Fergus}{2014}]{zeiler2014visualizing}
Zeiler, M.~D., and Fergus, R.
\newblock 2014.
\newblock Visualizing and understanding convolutional networks.
\newblock In {\em European conference on computer vision},  818--833.
\newblock Springer.

\bibitem[\protect\citeauthoryear{Zhai and Chen}{2018}]{zhai2018development}
Zhai, B., and Chen, J.
\newblock 2018.
\newblock Development of a stacked ensemble model for forecasting and analyzing
  daily average pm 2.5 concentrations in beijing, china.
\newblock {\em Science of the Total Environment} 635:644--658.

\end{thebibliography}
\bibliographystyle{aaai}

\clearpage

\end{document}